\documentclass[12pt,twocolumn]{IEEEtran}

\usepackage{graphicx}
\usepackage{amsthm}
\usepackage{amssymb,latexsym,multicol,caption}     
\usepackage{algorithm}  
\usepackage{algorithmic}
\usepackage{epsfig} 
\usepackage{epstopdf}
\usepackage[colorlinks, linkcolor=blue, anchorcolor=red, citecolor=red, urlcolor=black]{hyperref}
\usepackage{color}
\usepackage{multirow}
\usepackage{multicol}
\usepackage{subfigure}
\usepackage{placeins}
\usepackage{bm}
\usepackage{graphics}
\usepackage{footnote}
\usepackage{url}
\usepackage{indentfirst}
\usepackage{longtable}
\usepackage{array}
\usepackage{mdwmath}
\usepackage{mdwtab}
\usepackage{rotating}
\usepackage{color}
\usepackage{morefloats}
\usepackage{slashbox}
\usepackage{mcite}
\usepackage{eqnarray}
\usepackage{cite}
\usepackage[square, comma, sort&compress, numbers]{natbib}
\usepackage{float}
\usepackage{makecell}
\usepackage{amsmath}
\usepackage{threeparttable}
\usepackage{pifont}
\usepackage{cleveref}
\usepackage{balance}
\usepackage{flushend}

\crefname{equation}{}{}
\crefname{figure}{Fig.}{Figs.}
\crefname{table}{Table}{Tables}
\crefname{section}{Section}{Sections}
\crefname{lemma}{Lemma}{Lemmas}

\newcommand{\algref}[1]{Algorithm \ref{#1}}

\theoremstyle{plain}
\newtheorem{theorem}{Theorem}[]

\theoremstyle{definition}

\theoremstyle{remark}
\newtheorem*{rem}{Remark}

\ifCLASSINFOpdf

\else

\fi

\begin{document}
%
\title{ Kernel Two-Dimensional Ridge Regression for Subspace Clustering }

\author{
	Chong Peng, Qian Zhang, Zhao Kang, Chenglizhao Chen, and Qiang Cheng
	\thanks{C.P., Q.Z., and C.C. are with College of Computer Science and Technology, Qingdao University. Z.K. is with School of Computer Science and Engineering, University of Electronic Science and Technology of China. Q.C. is with the Institute of Biomedical Informatics \& Department of Computer Science, University of Kentucky. 
	C.C. is the corresponding author. Contact information: cclz123@163.com}
}

\maketitle

\begin{abstract}
Subspace clustering methods have been widely studied recently. When the inputs are 2-dimensional (2D) data, existing subspace clustering methods usually convert them into vectors, which severely damages inherent structures and relationships from original data. In this paper, we propose a novel subspace clustering method for 2D data. It directly uses 2D data as inputs such that the learning of representations benefits from inherent structures and relationships of the data. It simultaneously seeks image projection and representation coefficients such that they mutually enhance each other and lead to powerful data representations. An efficient algorithm is developed to solve the proposed objective function with provable decreasing and convergence property. Extensive experimental results verify the effectiveness of the new method. 

\end{abstract}

\IEEEpeerreviewmaketitle

\section{Introduction}
\label{sec_intto}

High-dimensional data are ubiquitous and commonly used in various real-world applications such as computer vision and image processing. 
Often times, such data have latent low-dimensional structures rather than uniformly distributed.
To illustrate this, we show a simple example in \cref{fig_subspace}. 
Such phenomenons are often seen in real-world applications.
For example, face images lie in high-dimensional space, however they belong to a few number of subjects and form clear low-dimensional structures.
\begin{figure}[h]
	\centering
	{\includegraphics[width=0.6\columnwidth]{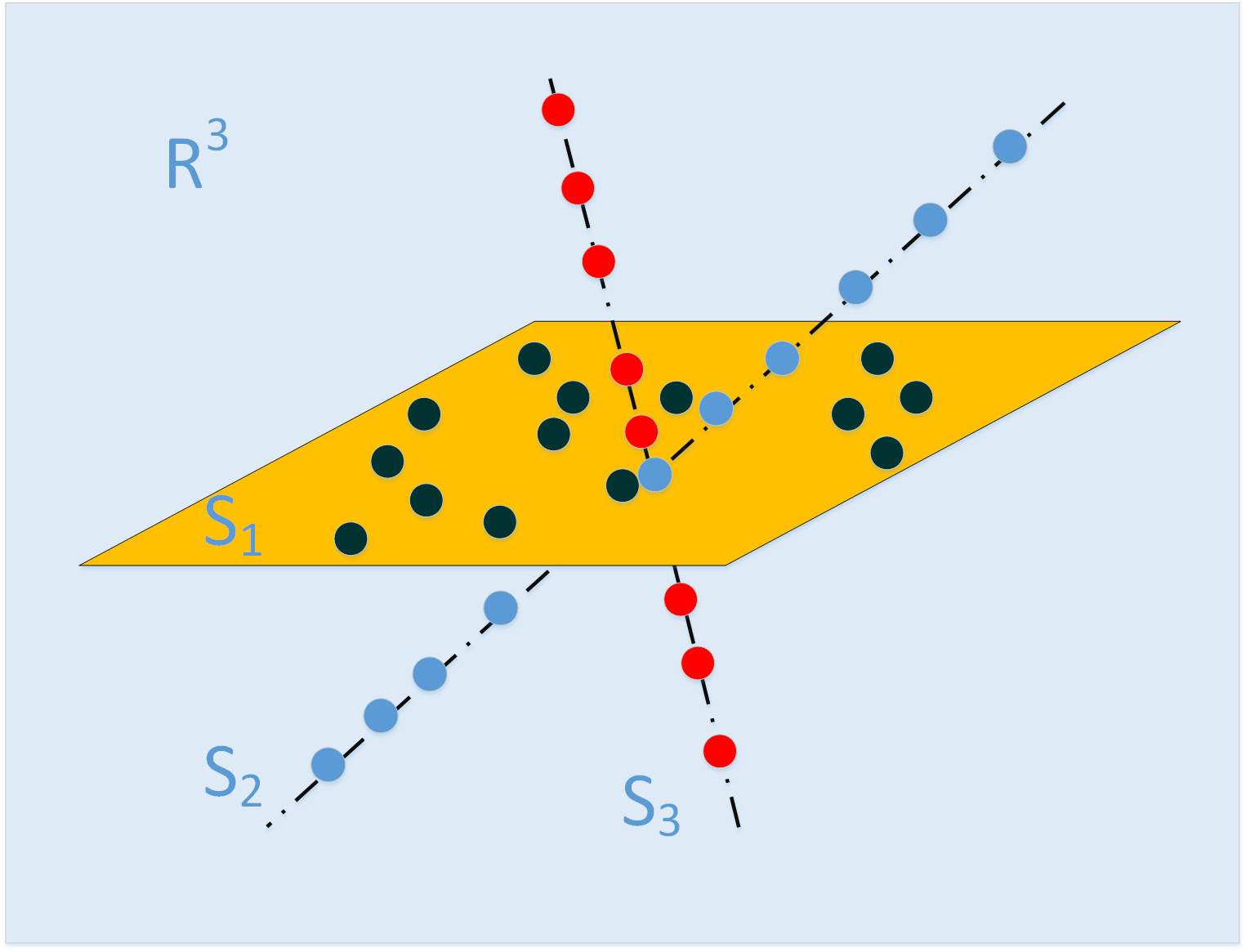}}
	\caption{ Example of high-dimensional data lying in low-dimensional subspaces. It is seen that rather than uniformly distributed in the 3-dimensional space, these data points lie on the union of two lines and one plane. }
	\label{fig_subspace}
\end{figure}
This inspires us to effectively represent high-dimensional data in low-dimensional subspaces \cite{peng2015subspace,liu2010robust}.
To recover such low-dimensional subspaces, it usually requires clustering the data into different groups.
Each of these groups can be fitted with a subspace and this procedure is known as subspace clustering or subspace segmentation. 

During the last decade, various types of subspace clustering algorithms have been developed.
These methods can be roughly categorized into 4 groups, 
including algebraic methods \cite{boult1991factorization,vidal2003generalized,ma2008estimation}, 
statistical methods \cite{gruber2004multibody,rao2010motion,ma2007segmentation}, 
iterative methods \cite{ho2003clustering,zhang2009median}, 
and spectral clustering based methods \cite{CHEN2018107,SUI2019261,CHEN2020107441}; see \cite{vidal2011subspace} for a review. 
%
%
%
%
%
%
%
Among them, spectral clustering-based methods have been popular with great success. 
Typical methods such as low-rank representation (LRR) \cite{liu2013robust,favaro2011closed} and sparse subspace clustering (SSC) \cite{elhamifar2013sparse} have drawn considerable attentions due to their efficiencies and elegant theories. 
The basic idea of LRR and SSC is the self-expressive property of the data, which suggests that each example of the data can be represented by the data as a dictionary.
With specific structure requirements on the representation matrix, the learned representation coefficient matrix by LRR and SSC admit low rankness and sparsity, respectively.
In the ideal case, such low-rank or sparse structure clearly shows group information of the data.
Recent work also attempts to merge the advantages of simultaneous low-rank and sparse learning with both low-rank and sparse regularization terms \cite{BRBIC2018247}.


It is pointed out that the nuclear norm is not accurate for rank approximation, which makes LRR less efficiency in learning accurate structure of the data \cite{peng2015subspace}. 
To overcome this drawback, recent works develop various more accurate non-convex approximations to the rank function, such as the log-determinant rank approximation, which significantly improves the learning performance \cite{peng2015subspace}. 
Some studies demonstrate the importance of feature learning for subspace clustering \cite{peng2016feature,patel2013latent}. 
For example, \cite{peng2016feature,peng2017integrating} seek a low-rank representation with respect to a subset of features, which alleviates the importance of rank approximations; 
\cite{patel2013latent} seeks a sparse representation of projected data in a latent low-dimensional space such that hidden structures of the data provide useful information. 
To consider nonlinear structures of the data, various approaches have been attempted. 
For example, grpah Laplacian is introduced to LRR \cite{liu2014enhancing}, which accounts for nonlinear relationships of the data on manifold; 
kernel is adopted in LRR \cite{xiao2016robust} and SSC \cite{patel2014kernel}, respectively, which seeks sparse representation of the data in nonlinear feature space. 
Other types of representation are also shown to be successful in subspace clustering, such as thresholding ridge regression \cite{peng2015robust} and simplex representation \cite{Xu2019Scaled}.
Other than dealing with noise effect with data reconstruction term, \cite{peng2015robust} alleviates the noise effect by vanishing the small values in the coefficients obtained from a ridge regression model. 
Simplex representation is similar to ridge regression, which seeks a representation matrix of the data with additional constraints on it \cite{Xu2019Scaled}.

Subspace clustering is often used in problems that deal with 2-dimensional (2D) data, with each example being a matrix.
Unfortunately, these methods usually suffer from a common issue when dealing with such data.
When handling 2D data, these methods usually convert all examples of the data to vectors in a pre-processing step, which severely damages inherent structural information of the data. 
This strategy omits the inherent structures and correlations of the original data which are essentially important, 
and building models with vectored data is not effective to filter the noise, occlusions or redundant information \cite{fu2016tensor}. 
To better handle 2D data, tensor-based methods have been considered in many areas, 
such as non-negative tensor factorization \cite{Benaroya2018Binaural}, tensor robust principal component analysis \cite{LIU2020107252,lu2016tensor}, tensor subspace learning \cite{Pan2019Tensor,zhang2015low}, etc. 
For tensor methods, tensor decomposition is often needed, where the main techniques include candecomp/parafac decomposition (CPD) and Tucker decomposition (TD). 
Tensor-based subspace clustering methods usually involve flatting and folding operations, which may not measure the true structures of the data \cite{Zhou2019Tensor}.
More importantly, tensor methods usually suffer from the following major issues: 
1) for CPD-based methods, it is generally NP-hard to compute the CP rank \cite{lu2016tensor,kolda2009tensor}; 
2) TD is not unique \cite{kolda2009tensor}; 
3) the application of a core tensor and a high-order tensor product would incure information loss of spatial details \cite{letexier2008noise}.

Besides tensor-based methods, some other approaches have been attempted to handle 2D data in many areas, 
such as 2-dimensional principal component analysis (2DPCA) \cite{yang2004two},
2-dimensional semi-nonnegative matrix factorization \cite{Peng2020TwoDimensionalSM}, and nuclear norm-based 2DPCA \cite{zhang2014nuclear}. 
2DPCA uses a projection matrix to extract the most representative spatial information from 2D data, which inspires us to recover low-dimensional subspaces of 2D data with such features.
Thus, to overcome the above mentioned key drawbacks of the current subspace clustering methods,
we propose a novel method for 2-dimensional data, which directly uses a projection matrix on the original 2D data, such that the rich structural information of the data can be maximally used in the learning process.
%
We briefly summarize key contributions of the paper as follows:
1) Unlike existing methods that perform vectorization to 2D data in a pre-processing step, we propose to learn a 2D projection matrix such that the most expressive structural information is retained in the spanned subspaces; 
2) The learning of projection and construction of representation are seamlessly integrated, such that these two tasks mutually enhance each other and lead to powerful representation; 
3) Kernel method for 2D data is introduced to our model, which explicitly considers nonlinear structures of the data;
4) Efficient optimization algorithm is developed with provable convergence guarantee;
5) The algorithm does not rely on augmented Lagrangian multiplier (ALM) type optimization as existing methods usually do, thus we do not need to introduce additional parameters in ALM framework;
6) Extensive experiments confirm the effectiveness of our method.


The rest of this paper is organized as follows. We briefly review some closely related methods in \cref{sec_related}. 
Then we introduce the proposed method, develop its optimization to obtain the representation matrix, and present how to perform clustering using the learned representation matrix in \cref{sec_proposed}.
We conduct extensive experiments to testify the effectiveness of the proposed method in \cref{sec_experiment}.
Finally, we conclude the paper in \cref{sec_conclusion}.

\section{Related Work}
\label{sec_related}
In this section, we briefly review some closely related subspace clustering methods.

Given the data matrix $A=[a_1,\cdots,a_n]\in\mathcal{R}^{d\times n}$ with each sample $a_i\in\mathcal{R}^d$, 
LRR seeks a low-rank representation of the data with the following minimization problem:
\begin{equation}
\min_{Z} \|A-AZ\|_{2,1} + \lambda \|Z\|_*,
\end{equation}
where $\|\cdot\|_{2,1}$ is the sum of column-wise $\ell_2$ norms of a matrix, $\|\cdot\|_*$ is the nuclear norm, $\lambda$ is a balancing parameter, and $Z\in\mathcal{R}^{n\times n}$ is the representation to be sought. Instead of seeking a low-rank representation, the SSC assumes sparse representation of the data, which leads to the following:
\begin{equation}
\begin{aligned}
&	\min_{Z,S,E} \|E\|_F^2 + \lambda \|Z\|_1 + \gamma \|S\|_1,	\\
&	s.t. \quad A = AZ + S + E, \textit{diag}(Z) = 0,
\end{aligned}
\end{equation}
where $\gamma$ is a balancing parameter and the constraint $\textit{diag}(Z) = 0$ avoids trivial solution to SSC.
The above models seek the representation of the data with the assumption of self-expressiveness. 
Various developments have been made based on LRR and SSC, such as nonlinear extensions \cite{Yin2016Laplacian,patel2014kernel} and feature integration approaches \cite{liu2011latent}.

\section{Kernel Two-Dimensional Ridge Regression}
\label{sec_proposed}

In this section, we will develop a new subspace clustering model based on ridge regression.
In the following of this section, we will present its formulation, optimization, and the clustering algorithm, respectively.

\subsection{Formulation of Kernel Two-Dimensional Ridge Regression}
Ridge regression-based data representation has been shown successful for high-dimensional data in both supervised \cite{peng2015robust} and unsupervised learning problems \cite{peng2020discriminative}.
For a collection of examples $\{X_i\}_{i=1}^{n}$ with each example $X_i\in\mathcal{R}^{a\times b}$ being a matrix,
inspired by \cite{peng2015robust,peng2020discriminative}, we seek a low-dimensional representation of the data with the following ridge regression model:
\begin{equation}
\label{eq_trr_2d}
\min_{Z} \sum_{i=1}^{n} \|X_i - \sum_{j=1}^{n} X_j Z_{ji} \|_F^2 + \lambda \|Z\|_F^2,
\end{equation}
where $\|\cdot\|_F$ is the Frobenius norm, and $\lambda\ge0$ is a balancing parameter. 
Here, unlike \cite{peng2015robust,elhamifar2009sparse} that vanish the diagonal elements of $Z$, we do not have such constraints due to the following two reasons: 
1) the example $X_i$ is in the intra-subspace of $X_i$ itself, thus it is meaningful if $Z_{ii}\not=0$; 
2) $\gamma>0$ excludes potentially trivial solutions such as $I_n$, where $I_n$ is an identity matrix of size $n\times n$. 

It is straightforward that \cref{eq_trr_2d} is equivalent to seeking the representation $Z$ with vectored data due to the nature of element-wise operation of the squared Frobenius norm. 
To retain inherent spatial information of the data in the learning process, we introduce a projection vector $p\in\mathcal{R}^b$, i.e., a direction, 
which projects the data to a subspace in which the most expressive 2D feature of the data is retained.
That is, each example $X_i$ is projected as $X_ipp^T$ to the subspace spanned by $p$. 
To mutually enhance the learning tasks of the projection and representation, we propose to simultaneously seek the representation with projected data as follows:
\begin{equation}
\label{eq_obj_pp}
\begin{aligned}
\min_{p^Tp = 1,Z} & \sum_{i=1}^{n}\Big\| X_ipp^T - \sum_{j=1}^{n} X_j pp^T z_{ji} \Big\|_F^2 \\
& + \lambda \sum_{i=1}^{n}\|X_i - X_i pp^T \|_F^2 + \gamma \|Z\|_F^2,
\end{aligned}
\end{equation}
where $\gamma \ge 0$ is a balancing parameter. 
It is seen that the projection vector $p$ captures spatial information of the data and the representation is sought with the projected data and thus benefits from spatial information of the data. 
The first term of \cref{eq_obj_pp} can be derived as $ \sum_{i=1}^n \| X_ipp^T - \sum_{j=1}^{n} X_j pp^T z_{ji} \|_F^2 = \sum_{i=1}^n \| (X_i - \sum_{j=1}^{n} X_j z_{ji} ) pp^T \|_F^2 = \sum_{i=1}^n \| (X_i - \sum_{j=1}^{n} X_j z_{ji} ) p \|_2^2 = \sum_{i=1}^n \| X_ip - \sum_{j=1}^{n} X_j p z_{ji} \|_2^2.$
Thus, \cref{eq_obj_pp} can be mathematically simplified as 
\begin{equation}
\label{eq_obj_p}
\begin{aligned}
\min_{p^Tp = 1,Z} & \sum_{i=1}^{n}\Big\| X_ip - \sum_{j=1}^{n} X_j p z_{ji} \Big\|_2^2 \\
& + \lambda \sum_{i=1}^{n}\|X_i - X_i pp^T \|_F^2 + \gamma \|Z\|_F^2.
\end{aligned}
\end{equation}
Usually, it is not enough to seek a single projection vector in real-world applications and multiple projection directions are often needed. 
Major information of the data may exist in several distinct subspaces and recovering multiple subspaces may allow us to better understand the data.
To seek multiple projection directions or feature subspaces, we define a projection matrix $P=[p_1,p_2,\cdots,p_r]\in\mathcal{R}^{b\times r}$ with $p_i$ being a projection direction satisfying that $p_i^Tp_i=1$ and $p_i^Tp_j=0$ for $i\not= j$. 
With $P$, we expand \cref{eq_obj_p} to construct the representation with simultaneous learning of multiple projection directions: 
\begin{equation}
\label{eq_obj_P}
\begin{aligned}
\min_{P^TP = I_r,Z} & \sum_{i=1}^{n}\Big\| X_iP - \sum_{j=1}^{n} X_j P z_{ji} \Big\|_F^2 \\
&+ \lambda \sum_{i=1}^{n}\|X_i - X_i PP^T \|_F^2 + \gamma \|Z\|_F^2,
\end{aligned}
\end{equation}
where $I_r$ is an identity matrix of size $r\times r$. 
It is seen that in \cref{eq_obj_P} the coefficient matrix $Z$ is constructed using the projected features $(X_jP)'s$ or equivalently the projected samples $(X_jPP^T)'s$, which contain the most expressive information in the orthogonal subspaces $(p_jp_j^T)'s$ spanned by projection vectors $p_j's$. 
The projection in our model performs dimension reduction, for which we claim from the following two perspectives:
1) The original examples have size $a\times b$ and the projection reduces the size of examples to $a\times r$.
2) The original examples have $c=\min\{a,b\}$ 2D component features.
With the projection, it is seen that only up to $r$ 2D component features are used in the construction of representation matrix $Z$.
In this paper, we consider the number of 2D features as the dimension, thus the projection actually extracts most expressive 2D features of the data and performs dimension reduction.
For ease of representation, we define 
%
\begin{equation}
\label{eq_JP}
\mathcal{J} = \sum_{i=1}^{n} \Bigg\{\|X_iP - \sum_{j=1}^{n} X_j P z_{ji} \|_F^2 + \lambda \|X_i - X_i PP^T \|_F^2\Bigg\},
\end{equation}
and thus \cref{eq_obj_P} can be written as  
\begin{equation}
\label{eq_obj_linear}
\begin{aligned}
\min_{P^TP = I_r,Z} & \mathcal{J} + \gamma \|Z\|_F^2.
\end{aligned}
\end{equation}

Up to now, model \cref{eq_obj_linear} only considers the linear relationships of projected data in the Euclidean space. 
In real world problems, nonlinear relationships of data often exist and should be counted in data processing. 
To directly take nonlinear relationships of 2D data into consideration, we adopt the kernel approach for 2D data and develop a nonlinear model in remaining of this section. 
Inspired by \cite{nhat2007kernel}, we define nonlinear mappings of the data in the following. 
For a 2D example $M\in\mathcal{R}^{a\times b}$, we define $m_i \in \mathcal{R}^{a \times 1}$ 
to be its column instance vectors, i.e.,
%
\begin{equation}
M = \begin{bmatrix}
m_1 & \cdots & m_b
\end{bmatrix}.
\end{equation}
We define $\phi:\mathcal{R}^{a\times b} \rightarrow \mathcal{R}^{f_a \times b}$ with $f_a \ge a$ being a column-wise nonlinear mapping, such that it maps columns of a matrix to nonlinear space:
\begin{equation}
\label{eq_mapping_phi}
\phi(M) = \begin{bmatrix}
\phi(m_1) & \cdots & \phi(m_b)
\end{bmatrix},
\end{equation}
where $\phi(M) \in \mathcal{R}^{f_a \times b}$ and $\phi(m_i) \in \mathcal{R}^{f_a \times 1}$. 
For two matrices of the same size $U = [u_1,\cdots, u_b] \in\mathcal{R}^{a\times b}, V = [v_1,\cdots, v_b] \in\mathcal{R}^{a\times b}$, 
it is straightforward to obtain the following multiplications by simple algebra:
\begin{equation}
\label{eq_kernel_phi}
\begin{aligned}
\phi^T(U) \phi(V) = 
\begin{bmatrix}
\phi^T(u_1)\phi(v_1) & \cdots & \phi^T(u_1)\phi(v_b)	\\
\vdots  & \ddots & \vdots	\\
\phi^T(u_b)\phi(v_1) & \cdots & \phi^T(u_b)\phi(v_b)	
\end{bmatrix},
\end{aligned}
\end{equation}
where $\phi^T(\cdot)$ denotes $(\phi(\cdot))^T$ for simplicity, and $u_i$, $v_j$ are columns of $U$ and $V$, respectively. 
It is seen that each element of \cref{eq_kernel_phi} is inner product of mapped instance vectors and thus can be calculated by $ k(u_i,v_j) = \phi^T(u_i)\phi(v_j)$, 
where $k:\mathcal{R}^{a} \times \mathcal{R}^{a}\rightarrow \mathcal{R}$ is a kernel function. 
By defining $\mathcal{K}^{\phi}_{ij} = \phi^T(X_i)\phi(X_j) \in\mathcal{R}^{b\times b}$, we can see that $\mathcal{J}$ can be extended to its nonlinear version $\mathcal{J}^{\phi}$ in the kernel space:
\begin{equation}
\label{eq_JP_kernel}
\begin{aligned}
\mathcal{J}^{\phi} = &	\sum_{i=1}^{n} \Bigg\{\|\phi(X_i)P - \sum_{j=1}^{n} \phi(X_j) P z_{ji} \|_F^2 \\
&+ \lambda \|\phi(X_i) - \phi(X_i) PP^T \|_F^2 \Bigg\} \\
=	&	\sum_{i=1}^{n}\textbf{Tr}\Bigg\{ P^T\phi^T(X_i) \phi(X_i)P  \\
& -P^T\sum_{j=1}^{n} \phi^T(X_i) \phi(X_j) P z_{ji} \\
& \quad- \sum_{j=1}^{n} P^T\phi^T(X_j) \phi(X_i) z_{ji} P  \\
		\nonumber\end{aligned}\end{equation}\begin{equation}\begin{aligned}
& + \sum_{s=1}^{n}\sum_{t=1}^{n} P^T\phi^T(X_s)\phi(X_t) P  z_{si} z_{ti} \Bigg\}	\\
&	+ \lambda \sum_{i=1}^{n} \textbf{Tr}\Bigg\{ \phi^T(X_i) \phi(X_i) - \phi^T(X_i) \phi(X_i) PP^T \Bigg\}	\\
=	&	\sum_{i=1}^{n}\textbf{Tr}\Bigg\{ P^T \mathcal{K}^{\phi}_{ii} P -P^T\sum_{j=1}^{n} \mathcal{K}^{\phi}_{ij} P z_{ji} \\
&  - \sum_{j=1}^{n} P^T \mathcal{K}^{\phi}_{ji} z_{ji} P  + \sum_{s=1}^{n}\sum_{t=1}^{n} P^T \mathcal{K}^{\phi}_{st} P  z_{si} z_{ti} \Bigg\} \\
& + \lambda \sum_{i=1}^{n} \textbf{Tr}\Bigg\{ \mathcal{K}^{\phi}_{ii} - \mathcal{K}^{\phi}_{ii} PP^T \Bigg\}.	\\
\end{aligned}
\end{equation}
Therefore, by extending $\mathcal{J}$ to kernel version, we extend \cref{eq_obj_P} to the following nonlinear model, which is named Kernel Two-dimensional Ridge Regression (KTRR):
\begin{equation}
\label{eq_obj_kernel}
\begin{aligned}
\min_{P^TP = I_r, Z} \mathcal{J}^{\phi} + \gamma \|Z\|_F^2.
\end{aligned}
\end{equation}
It is seen that the representation $Z$ is sought with the nonlinear similarity matrices of the examples. 
It is worth pointing out that the integrated projection $P$ extracts spatial information of the data from the right side, i.e., in vertical direction.
It is straightforward to extend the above model by introducing another projection matrix $Q\in\mathcal{R}^{a\times r}$ to project the data from left side, 
such that spatial information from both vertical and horizontal directions can be retained.
However, the current model \cref{eq_obj_kernel} already provides us with the key idea and contribution of the paper, i.e., seeking representation with 2D features in nonlinear space, and extending \cref{eq_obj_kernel} with $Q$ is not within the main scope of this paper. 
Thus, in this paper, we focus on \cref{eq_obj_kernel} and do not fully expand the model to the bi-directional case.
We will discuss the optimization of \cref{eq_obj_kernel} in the following of this section. 


\subsection{Optimization of \cref{eq_obj_kernel} }
\label{sec_optimization}

In the above subsection, we have proposed a new subspace clustering model for 2D data. 
In this subsection, we will develop an alternating minimization algorithm for its optimization. 
Specifically, we alternatively solve the sub-problem associated with a single variable while keeping the others fixed.
We repeat the procedure until convergent.  
It is worth mentioning that the optimization does not rely on ALM type optimization and thus no additional parameters are needed as existing methods usually do. 
We regard this as an advantage because such parameters usually have effects on the solution and it takes efforts to tune such parameters.
The detailed optimization strategy is discussed as follows. 

\subsubsection{$P$-minimization}
The sub-problem associated with $P$-minimization is
\begin{equation}
\label{eq_sub_P}
\begin{aligned}
\min_{P^TP = I_r}	\mathcal{J}^{\phi}.
\end{aligned}
\end{equation}
It is seen that
\begin{equation}
\label{eq_sub_P_rewrite}
\begin{aligned}
\mathcal{J}^{\phi}  = &	\textbf{Tr}\Bigg\{ P^T (\sum_{i=1}^{n} \mathcal{K}^{\phi}_{ii}) P \Bigg\} \\
& - \textbf{Tr}\Bigg\{P^T (\sum_{i=1}^{n}\sum_{j=1}^{n} (\mathcal{K}^{\phi}_{ij} + \mathcal{K}^{\phi}_{ji})z_{ji}) P\Bigg\}	\\
&	+ \textbf{Tr}\Bigg\{ P^T (\sum_{s=1}^{n}\sum_{t=1}^{n}\mathcal{K}^{\phi}_{st}  z_{\bar{s}} z_{\bar{t}}^T) P \Bigg\} \\
& + \lambda \textbf{Tr}\Bigg\{ \sum_{i=1}^{n} \mathcal{K}^{\phi}_{ii} - \sum_{i=1}^{n} \mathcal{K}^{\phi}_{ii} PP^T \Bigg\}	\\
=	&	\textbf{Tr}\Bigg\{ P^T \Big( (1-\lambda)\mathcal{H}^{\phi}_1 
+ \mathcal{H}^{\phi}_2 - \mathcal{H}^{\phi}_3 \Big) P  \Bigg\} + \xi^{\phi},
\end{aligned}
\end{equation}
where we define
\begin{equation}
\label{eq_sub_H_phi}
\begin{aligned}
\mathcal{H}^{\phi}_1 = &	\sum_{i=1}^{n} \mathcal{K}^{\phi}_{ii},\\
\mathcal{H}^{\phi}_2 = &	\sum_{s=1}^{n}\sum_{t=1}^{n}\mathcal{K}^{\phi}_{st}  z_{\bar{s}} z_{\bar{t}}^T
= \sum_{i=1}^{n}\sum_{j=1}^{n}\mathcal{K}^{\phi}_{ij}  z_{\bar{i}} z_{\bar{j}}^T,	\\
\mathcal{H}^{\phi}_3 = &	\sum_{i=1}^{n}\sum_{j=1}^{n} (\mathcal{K}^{\phi}_{ij} + \mathcal{K}^{\phi}_{ji})z_{ji}, \\
\xi^{\phi} = &	\lambda \textbf{Tr} \Big\{ \sum_{i=1}^{n} \mathcal{K}^{\phi}_{ii} \Big\}.
\end{aligned}
\end{equation}
Here, $z_{\bar{s}}$ and $z_{\bar{t}}$ denote the $s$-th and $t$-th rows of matrix $Z$, respectively. 
It is easy to check that the matrices $\mathcal{H}^{\phi}_1$, $\mathcal{H}^{\phi}_2$, and $\mathcal{H}^{\phi}_3$ defined in \cref{eq_sub_H_phi} are real symmetric. 
Hence, $(1-\lambda)\mathcal{H}^{\phi}_1 + \mathcal{H}^{\phi}_2 - \mathcal{H}^{\phi}_3$ is real symmetric and $P$ can be obtained by performing the standard eigenvalue decomposition:
\begin{equation}
\label{eq_sol_P}
P = \textbf{eig}_r \big( (1-\lambda)\mathcal{H}^{\phi}_1 + \mathcal{H}^{\phi}_2 - \mathcal{H}^{\phi}_3 \big),
\end{equation}
where $\textbf{eig}_r(\cdot)$ is an operator that returns eigenvectors of the input matrix that are associated with its $r$ smallest eigenvalues.

\subsubsection{$Z$-minimization}
Fixing $P$, the $Z$-minimization problem is 
\begin{equation}
\label{eq_sub_Z}
\begin{aligned}
\min_{Z} \mathcal{J}^{\phi} + \gamma \|Z\|_F^2.
\end{aligned}
\end{equation}

To simplify the notation of $Z$-minimization, we define an operator $\bar{\phi}(\cdot)$ such that $\bar{\phi}_P(X_i)$ and $\bar{\phi}_P(\bf{X})$ are defined as 
\begin{equation}
\bar{\phi}_P(X_i) = 
\begin{bmatrix}
\phi(X_i)p_1 \\ \vdots \\ \phi(X_i)p_r
\end{bmatrix} 
\in \mathcal{R}^{f_a r \times 1},
\end{equation}
and
\begin{equation}
\bar{\phi}_P(\textbf{X}) = 
\begin{bmatrix}
\bar{\phi}_P(X_1) & \cdots & \bar{\phi}_P(X_n)
\end{bmatrix} 
\in \mathcal{R}^{f_a r \times n}.
\end{equation}
Then it is seen that \cref{eq_sub_Z} can be mathematically derived as 
\begin{equation}
\label{eq_JP_phi_Z}
\begin{aligned}
& \mathcal{J}^{\phi} + \gamma \|Z\|_F^2 \\
= & \sum_{i=1}^{n} \Big\|\phi(X_i)P - \sum_{j=1}^{n} \phi(X_j) P z_{ji} \Big\|_F^2 + \gamma \|Z\|_F^2	\\
=	&	\sum_{i=1}^{n} \Big\|\bar{\phi}_P(X_i) - \sum_{j=1}^{n} \bar{\phi}_P(X_j) z_{ji} \Big\|_F^2	 + \gamma \|Z\|_F^2 \\
=	& \Big\| \bar{\phi}_P(\textbf{X}) - \bar{\phi}_P(\textbf{X}) Z \Big\|_F^2 + \gamma \|Z\|_F^2.
\end{aligned}
\end{equation}
It is seen that the $Z$-subproblem is quadratic and convex, which admits closed-form solution with its first-order optimality condition. 
Hence, $Z$ is solved by
\begin{equation}
\label{eq_sol_Z}
Z 	=	\Bigg( \bar{\phi}_P^T(\textbf{X}) \bar{\phi}_P(\textbf{X}) 
+ \gamma I_n \Bigg)^{-1} 
\Bigg( \bar{\phi}_P^T(\textbf{X}) \bar{\phi}_P(\textbf{X}) 
\Bigg).
\end{equation}

To explicitly expand \cref{eq_sol_Z} and give the precise solution of $Z$, we define the matrix $\bar{\mathcal{K}}^{\phi}\in\mathcal{R}^{n\times n}$ as follows:
\begin{equation}
\label{eq_bK_phi}
\begin{aligned}
\bar{\mathcal{K}}^{\phi} = &	\bar{\phi}_P^T (\textbf{X})\bar{\phi}_P(\textbf{X})	\\
=	&	\begin{bmatrix}
\bar{\phi}_P^T (X_1) \\ \vdots \\ \bar{\phi}_P^T (X_n)
\end{bmatrix}
\begin{bmatrix}
\bar{\phi}_P(X_1) & \cdots & \bar{\phi}_P(X_n)
\end{bmatrix}	
\\ =	&	\begin{bmatrix}
\sum_{s=1}^{r} p_s^T \mathcal{K}^{\phi}_{11} p_s & \cdots & \sum_{s=1}^{r} p_s^T \mathcal{K}^{\phi}_{1n} p_s	\\
\vdots  & \ddots & \vdots	\\
\sum_{s=1}^{r} p_s^T \mathcal{K}^{\phi}_{n1} p_s & \cdots & \sum_{s=1}^{r} p_s^T \mathcal{K}^{\phi}_{nn} p_s	
\end{bmatrix}	\\		\nonumber\end{aligned}\end{equation}\begin{equation}\begin{aligned}
=	&	\begin{bmatrix}
\textbf{Tr}(P^T\mathcal{K}^{\phi}_{11}P) & \cdots & \textbf{Tr}(P^T\mathcal{K}^{\phi}_{1n}P)	\\
\vdots  & \ddots & \vdots	\\
\textbf{Tr}(P^T\mathcal{K}^{\phi}_{n1}P) & \cdots & \textbf{Tr}(P^T\mathcal{K}^{\phi}_{nn}P)	
\end{bmatrix}.
\end{aligned}
\end{equation}
Incorporating \cref{eq_bK_phi} into \cref{eq_sol_Z}, we obtain the solution of $Z$ with explicit expression
\begin{equation}
\label{eq_sol_Z_K}
\begin{aligned}
Z = ( \bar{\mathcal{K}}^{\phi} + \gamma I_n )^{-1}(\bar{\mathcal{K}}^{\phi}).
\end{aligned}
\end{equation}
To be clearer, we summarize the optimization steps in \algref{alg_optimization}.
Regarding the optimization of KTRR, we have the following theorem to guarantee the convergence.
\begin{theorem}
	Denote the objective function of \cref{eq_obj_kernel} as $g(P,Z)$, then its value sequence $\{ g(P^t,Z^t) \}_{t=1}^{\infty}$ is decreasing under the update rules of \cref{eq_sol_P,eq_sol_Z_K} and converges.
\end{theorem}

\begin{proof}
	According to the optimization of $P$ and $Z$, it is easy to see that
	\begin{equation}
	g(P^{t},Z^{t}) \le g(P^{t+1},Z^{t}) \le
	g(P^{t+1},Z^{t+1}),
	\end{equation}
	hence \cref{eq_obj_kernel}, i.e., the value sequence $\{ g(P^t,Z^t) \}$ is decreasing under the update rules of \cref{eq_sol_P,eq_sol_Z_K}. Moreover, it is straightforward to verify the nonnegativity of $\{ g(P^t,Z^t) \}$ by the definition of $g(P,Z)$ in \cref{eq_obj_kernel}, hence $\{ g(P^t,Z^t) \}$ is bounded and thus converges.
\end{proof}

\begin{rem}
	We analyze the time complexity of KTRR as follows.
	To compute the kernel matrices $\mathcal{K}^{\phi}_{ij}$, we need $O(ab^2)$ for each and thus $O(n^2ab^2)$ complexity for all.
	At each step, the complexity comes from the calculation of $P^t$ and $Z^t$. 
	According to \cref{eq_sub_H_phi,eq_sol_P}, it takes $O(n^3 + n^2 b^2 + b^2 r)$ operations to solve $P$-subproblem per iteration.
	For $Z^t$-updating, it takes $O(n^2 br)$ operations to obtain $\bar{\mathcal{K}}^{\phi}$ in \cref{eq_bK_phi} and $O(b^3)$ operations to solve \cref{eq_sol_Z_K}.
	Thus, the overall complexity per each iteration of KTRR is $O(n^3 + n^2 b^2 + b^2 r + n^2 br + b^3) = O(n^3+n^2b^2+b^3)$. 
	
\end{rem}

{ \scriptsize 
	\begin{algorithm}[!tb]
		\algsetup{linenosize=\small }  \small
		\caption{ Solving \cref{eq_obj_kernel}: Kernel Two-dimensional Ridge Regression (KTRR) } 
		\vspace{1mm}
		\begin{algorithmic}[1] 
			\STATE \textbf{Input}: $\textbf{X}$, $\lambda$, $\gamma$, $\epsilon$ (convergent tolerance), $t_{max}$
			\STATE \textbf{Initialize:} $Z^0$, $P^0$, $t=0$.
			\STATE Construct kernel matrices $\mathcal{K}^{\phi}_{ij}$. 
			\REPEAT
			\STATE Update $P^t$ by \cref{eq_sol_P}.
			\STATE Update $Z^t$ by \cref{eq_sol_Z}.
			\STATE $t=t+1$.
			\UNTIL $t\geq t_{max}$ or $\{\mathcal{J}^{\phi}(P^t,Z^t)\}$ converges
			\STATE \textbf{Output}: $Z$, $P$
			\vspace{1mm}
		\end{algorithmic} 
		\label{alg_optimization}
	\end{algorithm}
}

\subsection{ Subspace Clustering Algorithm via KTRR }
After we obtain the representation matrix $Z$ by solving \cref{eq_obj_kernel}, we construct an affinity matrix $\textbf{A}$ in a post-processing step as commonly done for many spectral clustering-based subspace clustering methods \cite{peng2015subspace,liu2013robust}. 
Following \cite{peng2015subspace,liu2013robust}, we construct $\textbf{A}$ with the following steps:
\begin{itemize}
	\item[1)] Let $Z = U\Sigma V^{T}$ be the skinny SVD of $Z$. Define $\bar{Z} = U\Sigma^{1/2}$ to be the weighted column space of $Z$.
	\item[2)] Obtain $\bar{U}$ by normalizing each row of $\bar{Z}$.
	\item[3)] Construct the affinity matrix $\textbf{A}$ as $[\textbf{A}]_{ij}=\left(|[\bar{U}\bar{U}^{T}]_{ij}|\right)^{\phi}$,
	where $\phi\ge 1$ controls the sharpness of the affinity matrix between two data points\footnote{In this paper, we follow \cite{liu2010robust} and set $\phi=4$ for fair comparison. }.
\end{itemize}
Subsequently, we perform Normalized Cut (NCut) \cite{shi2000normalized} on $\textbf{A}$ in a way similar to \cite{agarwal2004k,peng2015subspace}. We will present the detailed experimental results in the following section.

\section{Experiment}
\label{sec_experiment}
In this section, we conduct extensive experiments to verify the effectiveness of the proposed method. 
In particular, we compare our method with several state-of-the-art subspace clustering algorithms, 
including LRR \cite{liu2013robust}, LapLRR \cite{liu2014enhancing}, SCLA \cite{peng2015subspace}, SSC \cite{elhamifar2013sparse}, S$^{3}$C \cite{li2015structured}, TLRR \cite{Zhou2019Tensor}, SSRSC \cite{Xu2019Scaled}, and DSCN \cite{pan2017deep}.
Seven data sets are used in our experiments, including Jaffe \cite{lyons1998japanese}, PIX \cite{hond1997distinctive}, Yale \cite{belhumeur1997eigenfaces},
Opticalpen, Alphadigit, ORL \cite{samaria1994parameterisation}, and PIE.
Three evaluation metrics are adopted in the experiments, including clustering accuracy, normalized mutual information (NMI), and purity, 
whose detailed information can be found in \cite{peng2018integrate,peng2017integrating}.
In rest of this section, we will introduce the subspace clustering methods, benchmark data sets, and detailed clustering performance and analysis, respectively.  
For purpose of re-productivity, we will provide our code at xxx (available after acceptance).

\subsection{Dataset}
For the data sets used in our experiments, we show some examples in \cref{fig_examples}. We briefly describe these data sets as follows:
\textbf{1) Yale}. It contains 165 gray scale images of 15 persons with 11 images of size 32$\times$32 per person. 
\textbf{2) JAFFE}. 10 Japanese female models posed 7 facial expressions and 213 images were collected. Each image has been rated on 6 motion adjectives by 60 Japanese subjects. 
\textbf{3) PIX}. 100 gray scale images of $100\times100$ pixels from 10 objects were collected.
\textbf{4) Alphadigit} data set is a binary data set, which collects handwritten digits 0-9 and letters A-Z. Totally, there are 36 classes and 39 samples for each class, of which each example has size of 20$\times$16 pixels. 
\textbf{5) Opticalpen} collects hand-written pen digits of 0-9. Totally, there are 1797 images of size 8$\times$8 in this data set. 
\textbf{6) ORL} contains face images of size $32\times 32$ pixels from 40 individuals. Each individual has 10 images taken at different times, with varying facial expressions, facial details, and lighting conditions.
\textbf{7) PIE} has face images of 68 persons with different poses, illumination conditions, and expressions. For each person, we select the first 5 images. All images are resized to $32\times 32$ pixels.

\begin{figure*}[!]
	\centering
	{\includegraphics[width=2\columnwidth]{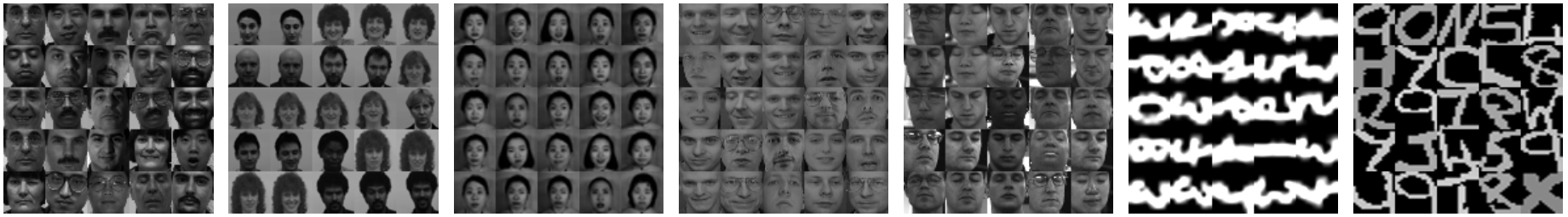}}
	\caption{ Examples of data sets used in our experiments. From left to right are example from Yale, PIX, Jaffe, ORL, PIE, Opticalpen, and Alphadigit data sets, respectively. }
	\label{fig_examples}
\end{figure*}

\subsection{Methods in Comparison}
\label{sec_methods}
To evaluate the performance of our method, we compare it with several state-of-the-art subspace clustering methods. For the baseline methods and KTRR, we briefly describe them as follows:
\begin{itemize}
	\item LRR seeks a low-rank representation of the data by minimizing the nuclear norm of the representation matrix. For its balancing parameter, we vary it within the set of $\{0.001,0.01,0.1,1,10,100,1000\}$;
	\item LapLRR is a nonlinear extension of the LRR, which exploits nonlinear relationships of the data on manifold \cite{Yin2016Laplacian}. 
	We follow \cite{peng2018integrate} and keep 5 neighbors on the graph, where binary and radial basis function (RBF) kernel with radial varying in $\{0.001, 0.01, 0.1, 1, 10, 100, 1000\}$ are used, respectively; 
	\item SCLA is a non-convex variant of the LRR, which seeks a low-rank representation of the data by minimizing the non-convex log-determinant rank approximation of the representation matrix.
	For its balancing parameters that controls the sparsity of noise and low-rankness of the representation, we vary them within the set of $\{0.001,0.01,0.1,1,10,100,1000\}$; 
	\item SSC seeks a sparse representation of the data by minimizing the $\ell_1$ norm of the representation matrix. We tune the regularization parameters within the set of $\{0.001,0.01,0.1,1,10,100,1000\}$;
	\item S$^3$C is an extension of SSC, which seeks sparse representation of the data in latent space. Moreover, S$^3$C improves the clustering capability by considering nonlinear relationships of the data. For its balancing parameter , we set it within the set $\{0.001,0.01,0.1,1,10,100,1000\}$. For its parameter that balances the sparsity and nonlinear structure of the representation, we set it within the set $\{0.1,0.15,0.2,0.25\}$;
	\item TLRR seeks a low-rank representation of the tensor-type data, where it recovers a clean low-rank tensor while infering the cluster structur of the data. 
	For its balancing parameter, we vary it within the set of \{0.001,0.01,0.1,1,10,100,1000\};
	\item SSRSC recovers physically meaningful and more discriminative coefficient matrix by restricting the non-negativity of coefficients and constraining sum of the coefficient vectors up to a scalar less than 1. For its parameters including the sum of coefficient vectors, the penalty parameter of ADMM framework, and the iteration number, we follow the original paper and set them to be 0.5, 0.5, and 5, respectively. For its balancing parameter, we vary it within the set of \{0.001,0.01,0.1,1,10,100,1000\}.
	\item DSCN \cite{pan2017deep} constructs a representation matrix with deep neural networks, 
	where it maps given samples using explicit hierarchical transformations and simultaneously learns the reconstruction coefficients. 
	For the network, we conduct experiments with different kernel sizes and three-layer network depths. 
	The kernel size and network depth are chosen within the sets of \{[3,3,3], [5,5,3]\} and \{[10,20,30], [10,20,40], [20,30,40]\}, respectively.
	\item KTRR seeks the least square representation of the data with 2D features. Both RBF and polynomial kernels are used, where we set the radial and power parameters within the sets of $\{0.001, 0.01, 0.1, 1, 10, 100, 1000\}$ and $\{1,2,$ $3,4,5,8,10\}$, respectively. 
	We set the number of projections and vary other balancing parameters within the sets of $\{1,3,5,7,9\}$ and $\{0.001,$ $ 0.01, 0.1, 1,$ $ 10, 100, 1000\}$, respectively.
	
\end{itemize}

\subsection{Comparison of Clustering Performance}
\label{sec_exp_performance}
In this section, we present the detailed comparison of KTRR and baseline methods.
To provide more comprehensive evaluation of KTRR, we consdier conducting experiments in a way similar to \cite{peng2015subspace,peng2017nonnegative,peng2020nonnegative}.
Specifically, the experimental setting is as follows.
For each data set, we conduct experiments using its subsets with different number of clusters.
In particular, for a data set with a total number of $\bar{N}$ clusters, we consider its subsets with $N$ clusters, where $N$ may range in a set of values. 
For example, in ORL data, $\bar{N}=40$ and we consider its subsets with 5, 10, 15, 20, 25, 30, 35, and 40 clusters, respectively.
It is clear that there are $\frac{\bar{N}!}{(\bar{N}-N)!N!}$ possible subsets for a specific $N$ value and we randomly choose 10 of them in the experiment.
We report the results in \cref{tab_per_jaffe,tab_per_pix,tab_per_opticalpen,tab_per_alphadigit,tab_per_yale,tab_per_orl,tab_per_pie}, where average performance over the 10 subsets is reported with respect to each $N$ value.

\begin{table*}[htbp] 
\centering
\caption{Clustering Performance on Jaffe }
\resizebox{0.9\textwidth}{0.2\textwidth}{
}%
\\{\scriptsize For all tables, the top two performances are highlighted in {\color{red}red} and {\color{blue}blue}, respectively.}
\label{tab_per_jaffe}
\end{table*}

\begin{table*}[htbp] 
\centering
\caption{Clustering Performance on PIX }
\resizebox{0.9\textwidth}{0.2\textwidth}{
}
\label{tab_per_pix}
\end{table*}

\begin{table*}[htbp] 
	\centering
	\caption{Clustering Performance on ORL}
	\resizebox{0.9\textwidth}{0.18\textwidth}{
\begin{tabular}{ |c|c|c|c|c|c | c|c|c|c|}
			\hline
			\multirow{2}{1cm}{\centering N}& \multicolumn{9}{c|}{Accuracy}\\	
			\cline{2-10}
			\multirow{2}{1cm}{}
		& DSCN	& SSRSC& TLRR   	& LRR				& LapLRR				& S$^3$C			& SCLA 			& SSC			& proposed						\\ \hline
			5& 0.7640$\pm$0.0554 & 0.8400$\pm$0.1488 & 0.8404$\pm$0.0890	& 0.8260$\pm$0.1310	& 0.8380$\pm$0.1471		& 0.8120$\pm$0.1244	& \textcolor{blue}{0.8700$\pm$0.1287}	& 0.7126$\pm$0.1405	& \textcolor{red}{0.8860$\pm$0.1116}		\\
			10& 0.7380$\pm$0.0327& 0.7860$\pm$0.0652 & 0.7750$\pm$0.0822	& 0.7090$\pm$0.0833	& 0.7360$\pm$0.1117		& 0.7530$\pm$0.0763	& \textcolor{red}{0.8300$\pm$0.0389}& 0.6851$\pm$0.0697	&\textcolor{blue}{ 0.8160$\pm$0.0753}		\\
			15& \textcolor{red}{0.8988$\pm$0.0029}& 0.7687$\pm$0.0536 & 0.7657$\pm$0.0374	& 0.7500$\pm$0.0664	& 0.7467$\pm$0.0553		& 0.7687$\pm$0.0655	& 0.8053$\pm$0.0618	& 0.6861$\pm$0.0584	& \textcolor{blue}{0.8467$\pm$0.0683}		\\
			20& \textcolor{red}{0.8350$\pm$0.0010}& 0.7280$\pm$0.0377 & 0.7083$\pm$0.0231	& 0.7310$\pm$0.0445	& 0.7405$\pm$0.0409		& 0.7175$\pm$0.0371	& 0.7805$\pm$0.0478	& 0.6819$\pm$0.0418	& \textcolor{blue}{0.7985$\pm$0.0426}		\\
			25& \textcolor{red}{0.8912$\pm$0.0133}& 0.7428$\pm$0.0394 & 0.6833$\pm$0.0233	& 0.7472$\pm$0.0339	& 0.7104$\pm$0.0334		& 0.7336$\pm$0.0270	& 0.7636$\pm$0.0262	& 0.6700$\pm$0.0458	& \textcolor{blue}{0.8028$\pm$0.0465}		\\
			30& \textcolor{red}{0.9133$\pm$0.0023}& 0.7353$\pm$0.0322 & 0.6473$\pm$0.0156	& 0.7383$\pm$0.0388	& 0.7253$\pm$0.0298		& 0.7430$\pm$0.0397	& 0.7480$\pm$0.0271	& 0.6352$\pm$0.0384	&\textcolor{blue}{ 0.8120$\pm$0.0335}		\\
\end{tabular}
	}
	\label{tab_per_orl}
\end{table*}

\begin{table*}[htbp] 
	\centering
	\caption{Clustering Performance on Yale }
	\resizebox{0.9\textwidth}{0.22\textwidth}{
	}
	\label{tab_per_yale}
\end{table*}

\begin{table*}[htbp] 
	\centering
	\caption{Clustering Performance on Opticalpen }
	\resizebox{0.9\textwidth}{0.2\textwidth}{
	}
	\label{tab_per_opticalpen}
\end{table*}
\begin{table*}[htbp] 
	\centering
	\caption{Clustering Performance on Alphadigit}
	\resizebox{0.9\textwidth}{0.18\textwidth}{
}	
	\label{tab_per_alphadigit}
\end{table*}

\begin{table*}[htbp] 
	\centering
	\caption{Clustering Performance on PIE}
	\resizebox{0.9\textwidth}{0.16\textwidth}{
	}
	\label{tab_per_pie}
\end{table*}

Generally, we observe that KTRR achieves the leading performance among all methods. 
Particularly, we have the following observations:
1) KTRR has the best performance in all cases on Jaffe data set.
2) KTRR has the best performance in almost all cases on PIX, Yale, and PIE data sets. 
3) KTRR is the best on Alphadigit data set, where it obtains the top two performance in almost all cases among which more than half are the best. 
4) KTRR is the second best method on Opticalpen and ORL data sets with quite competitive performance.
For each observation, we provide detailed discussion and analysis in the following. 

\subsubsection{Observation 1)}
It is seen that KTRR can cluster Jaffe data set correctly in all cases, whereas the baseline methods cannot.
Besides KTRR, SSRSC has the best performance, where it achieves the top two performances in all cases. 
LRR, LapLRR, SCLA, and SSC are also very competitive on this data set with the averaged performance higher than 99\%.
However, these methods are less competitive on other data sets, which will be discussed in later sections.
It should be noted that although some methods show promising performance, KTRR is the only method that achieves 100\% accuracy in all cases.

\subsubsection{Observation 2)}
It is seen that KTRR achieves the top performances in 32 out of 36, 17 out of 21, and 24 out of 27 cases on Yale, PIE, and PIX data sets, respectively.
Moreover, KTRR also achieves the top second performances on these data sets.
For example, KTRR obtains the top second performances in 2, 2, and 3 cases on Yale, PIE, and PIX data sets, respectively,
which indicates that KTRR has the top two performances in 34 out of 36, 19 out of 21, and 27 out of 27 cases on these data sets, respectively. 
On these data sets, the most competing methods include DSCN, SSRSC, LRR, LapLRR, and SCLA.
Compared with these methods, KTRR improves averaged clustering accuracy, NMI, and purity by at least 7\%, 6\%, and 6\% on Yale data set. 
The improvement can be even more significant if we compare KTRR with each baseline method, respectively.
For example, we can see that KTRR improves the averaged NMI by about 10\% compared with SSRSC and DSCN.
On Yale data set, it is seen that LRR, LapLRR, and SCLA are comparable to each other and they obtain the top two best performances in 11, 10, and 9 out of 36 cases, respectively.
However, such kind of performances is still significantly inferior to KTRR.

On PIX and PIE data sets, the most competing methods include DSCN and S$^3$C. 
Similar observations to Yale data set can be found.
That is, the most competing methods show better performances to the other baseline ones, but inferior to KTRR. 
Moreover, methods such as SSRSC and S$^3$C do not always show competing performance on all these data sets, whereas KTRR is consistently the best. 
These observations indicate the superior performance of KTRR.

\subsubsection{Observation 3)}
On Alphadigit data set, KTRR achieves the highest, the top second, and the top third performances in 13 and 8, and 3 out of 24 cases, respectively.
It is seen that KTRR obtains more than half of the best and almost all of the top two performances on this data set. 
Among the baseline methods, DSCN, LRR, and SCLA achieve 6, 4, and 1 the top performances, respectively. 
Moreover, DSCN achieves the top second performances in 4 cases. 
Generally, DSCN is the second best method on Alphadigit data set, but its performance is less promising than KTRR. 
In general, we may conclude that KTRR outperforms DSCN, as well as the other methods on Alphadigit data set.

\subsubsection{Observation 4)}
On ORL data set, DSCN and KTRR are the most competitive methods.
It is seen that DSCN obtains the top two performances in 18 out of 24 cases, among which 15 are the best and 3 are the top second, respectively. 
KTRR obtains the top two performances in all cases, including 6 cases with the best performances. 
Moreover, in the average cases, DSCN outperforms KTRR in accuracy and purity by about 1-2\% whereas KTRR outperforms DSCN in NMI by about 4\%. 
Among the other methods, SCLA obtains the top two performances in 7 cases, which is observed to be the best.
These observations indicate that KTRR is competitive to DSCN while superior to the other baseline methods on ORL data set. 

On Opticalpen data set, SSRSC, LRR, and KTRR are the most competitive methods, among which SSRSC is the best.
It is observed that SSRSC achieves the best performances in 17 out of 27 cases, which suggests its superior performance to the other methods on Opticalpen data set.
Among the other methods, KTRR is the best, where it obtains the best and the top second performances in 3 and 13 cases, respectively.
Overall, SSRSC and KTRR achieve the top two performances in 18 and 16 cases, respectively.
Moreover, LRR has the top second performances in 8 cases but no best ones, showing inferior performance to KTRR.
These observations indicate that though KTRR is not the best on Opticalpen data set, it is rather competitive to SSRSC and superior to the other methods.

\subsubsection{Discussion}
It is observed although KTRR outperforms the other methods on Alphadigit data set, 
the improvements are relatively less significant than on other data sets such as Yale. 
Moreover, although KTRR has the best performances in several cases on Opticalpen data set, generally it is inferior to SSRSC on this data set.
One reasonable explanation is as follows. Alphadigit and Opticalpen data sets are pendigit images while the others are face images.
It is observed that pendigit images contain less structural information than face images.
Thus it is relatively more difficult to extract rich and useful structural information with the projection when constructing the representation matrix. 
However, KTRR still outperforms or is comparable to the baseline methods on these data sets.
In general, all algorithms have relatively better performance on ``easy" data sets such as Jaffe and PIX than the ``hard" ones such as ORL and Alphadigit.
The reason is that Jaffe and PIX data sets have less variations while the other data sets are more complicated.
For example, face images in PIE data sets may have different angle, facial expression, lighting conditions, and wearings;
some images in Alphadigit data set have similar shapes but belong to different categories, such as digit ``0" and letter ``O".
These properties of the data sets make the corresponding classification task more challenging.

In general, we can see that the baseline methods may obtain the best performances on some data sets, 
but they do not consistently show superior performance to KTRR on other data sets. 
For example, DSCN is the best method on ORL data set.
However, KTRR outperforms DSCN on the other data sets. 
These observations suggest effectiveness and superior clustering performance of the KTRR to the baseline methods.
In the following subsections, we will further evaluate KTRR with some more detailed tests.

\begin{figure}[!t]
	\centering
	{\includegraphics[width=1\columnwidth]{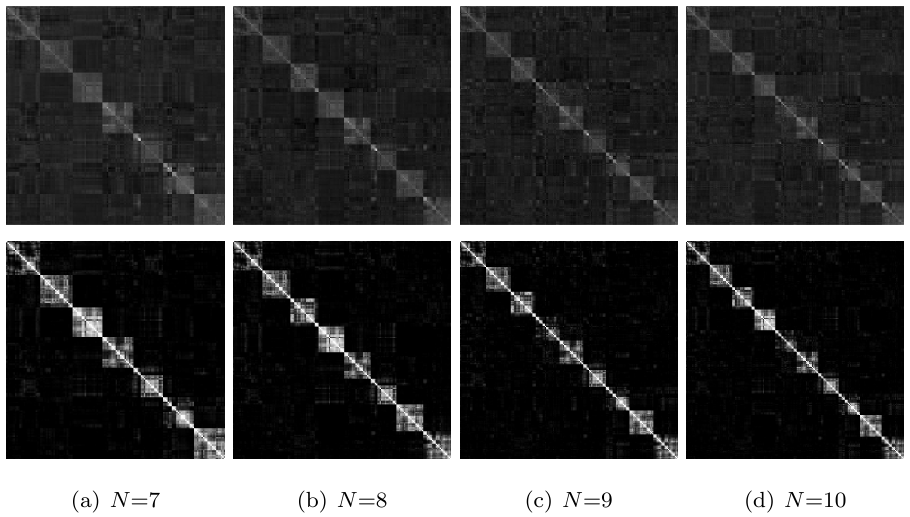}}
	\caption{ Example of learned representation matrix $Z$ (on the top) and the constructed affinity matrix $\textbf{A}$ (on the bottom) on Jaffe data. }
	\label{fig_Z}
\end{figure}

\subsection{Learned Representation}
In the above test, we have conducted extensive experiments to evaluate the the clustering performance of all methods, which has confirmed the effectiveness of KTRR. 
To better understand the clustering behavior of the KTRR, in this test, we visually show some examples of the learned representation matrix $Z$ as well as the constructed affinity matrix $\textbf{A}$ in the post-processing.
Without loss of generality, we show the matrices on Jaffe data, where we consider the cases of $N=7, 8, 9,$ and 10, respectively.
We visually show these matrices in \cref{fig_Z}.
It is seen that the learned representation matrices have clear block-diagonal structure, which clearly shows group information of the data.
The post-processing step makes the structured representation sharped, leading to even stronger structural effects. 
Hence, the proposed method performs clustering effectively with such representation matrices.

\subsection{Convergence Study}
\label{sec_exp_conv}
In \cref{sec_optimization}, we have theoretically analyzed the convergence of objective value. 
To better understand the convergence behavior of the proposed algorithm, we empirically show some results of convergence.
In this test, we use Jaffe and Alphadigit data sets for illustration.
To empirically testify the convergence of KTRR in objective value, without loss of generality, we fix $r=5, \lambda = 0.1, \gamma = 0.1$ and iterate the algorithm 50 times.
We plot the objective values in \cref{fig_conv_f}.  
It is observed that the proposed algorithm converges in objective value within a few number of iterations. 

\begin{figure}[!t]
	\centering
	{\includegraphics[width=1\columnwidth]{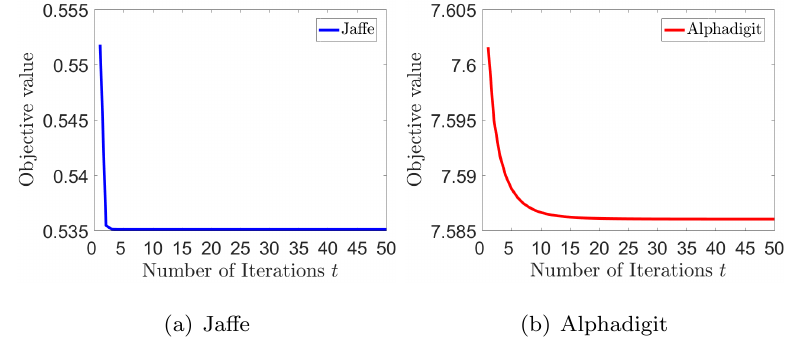}}
	\caption{ Examples of convergence curves of the objective value on Jaffe and Alphadigit data sets. Linear kernel is used and the other parameters are fixed as $r = 5$, $\lambda = 0.1$, and $\gamma = 0.1$. }
	\label{fig_conv_f}
\end{figure}

Moreover, since it is difficult to provide theoretical results on the convergence of variables, in this test we show some experimental results to verify this. 
To show the convergence of $\{Z_t\}$ and $\{P_t\}$, 
we show the plots of sequences $\{\|Z_{t+1}-Z_{t}\|_F\}_{t=0}^{\infty}$ and $\{\|P_{t+1}-P_{t}\|_F\}_{t=0}^{\infty}$, 
i.e., the difference of consecutive updates of variables.
We remain the above settings and show the results in \cref{fig_conv_zp}.
It is observed that the proposed algorithm converges within a few number of iterations in both $\{P_t\}$ and $\{Z_t\}$, which implies fast convergence of the proposed method in variable sequence. 
Similar convergence pattern can be observed on other data sets with various parameters. 
These observations suggest fast convergence and efficiency of KTRR and its potential applicability in real-world applications.

\begin{figure}[!h]
	\centering
	{\includegraphics[width=1\columnwidth]{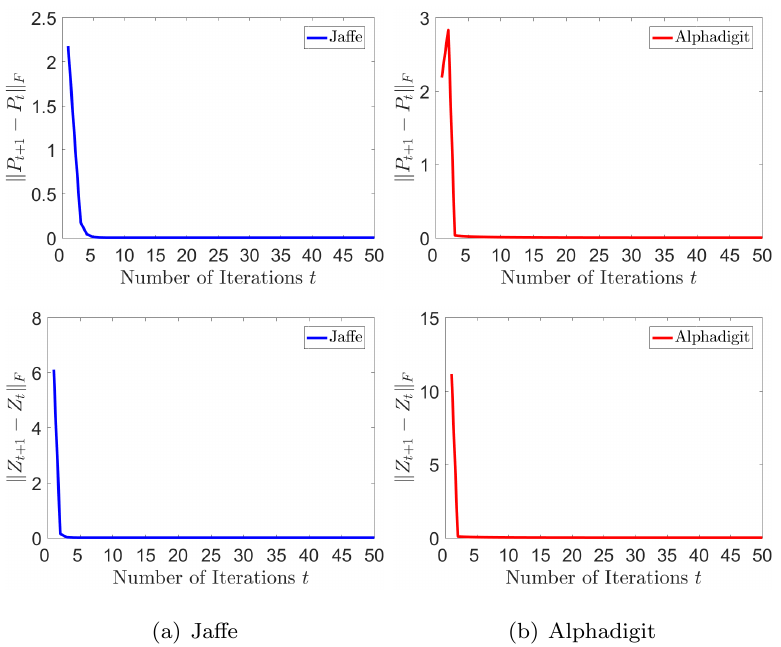}}
	\caption{ Examples of convergence curves of the variables on Jaffe and Alphadigit data sets. Linear kernel is used and the other parameters are fixed as $r = 5$, $\lambda = 0.1$, and $\gamma = 0.1$. }
	\label{fig_conv_zp}
\end{figure}

\subsection{Feature Extraction and Data Reconstruction}
In this subsection, we show some results on how the sought projection matrix works.
We use Yale data and adopt the linear kernel for illustration.
Without loss of generality, we fix $r = 30$, $\lambda = 0.01$, and $\gamma = 0.01$ and obtain the projection matrix $P$.
We show the extracted features and reconstructed examples by $P$ in \cref{fig_reconstruction}. 
It is seen that the key features of the examples can be captured with a few number of projection directions. 
These key features well reconstruct the original example, suggesting the effectiveness of the proposed method in feature extraction.

\begin{figure}[!t]
	\centering
	\subfigure[Example 1]{\includegraphics[width=1\columnwidth]{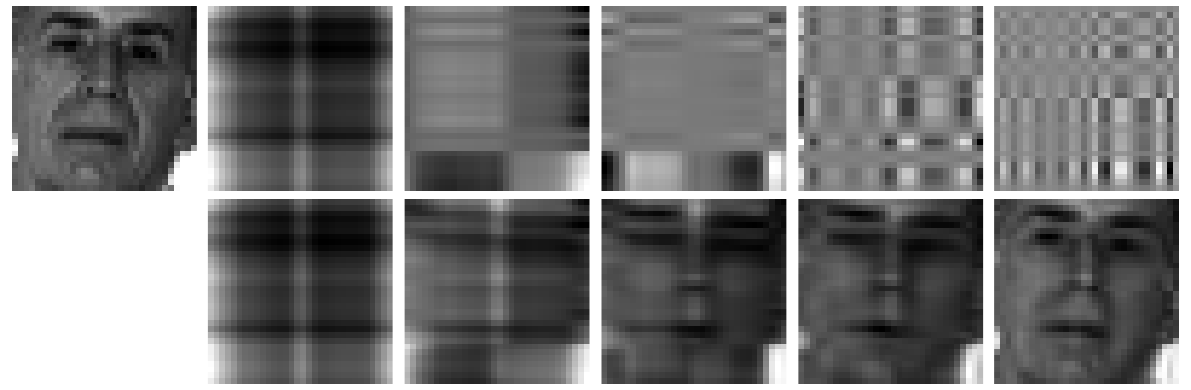}}
	\subfigure[Example 2]{\includegraphics[width=1\columnwidth]{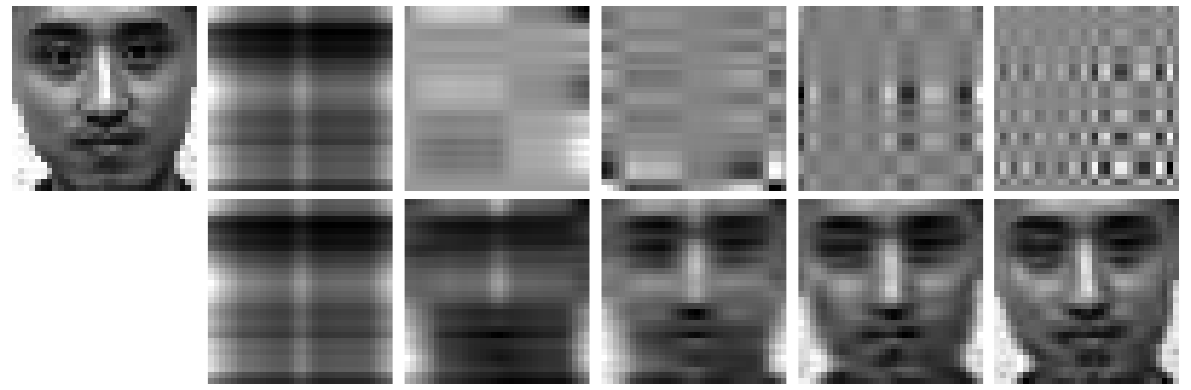}}
	\caption{ Examples of reconstructed examples on Yale data set. In each panel, the top left is the original sample image. For the rest, the top are the extracted $j$-th feature $Xp_jp_j^T$ while the bottom are the reconstructed image using the top $j$ features $\sum_{s=1}^{j}Xp_jp_j^T$. From left to right, $j=$ 1, 3, 5, 9, and 15, respectively. Linear kernel is used for reconstruction and the other parameters are fixed to be $\lambda = 1$, $\gamma = 0.01$, and $r=15$. }
	\label{fig_reconstruction}
\end{figure}

\begin{figure*}[!]
	\centering
	{\includegraphics[width=1.3\columnwidth]{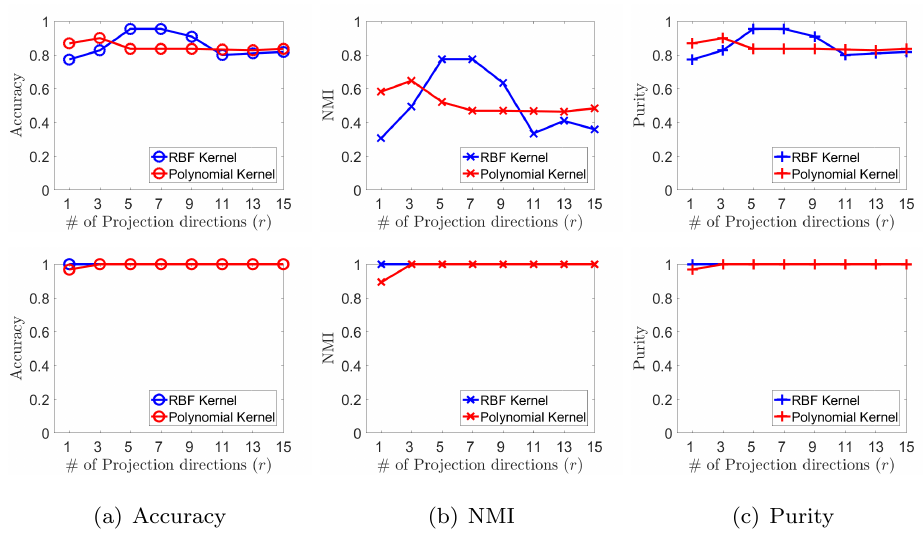}}
	\caption{ Examples of how the number of projections affects the performance of KTRR in accuracy, NMI, and purity on Yale (on the top) and Jaffe (on the bottom) data sets. For a specific $r$ value, we report the best performance by tuning all the other parameters in a grid-search scheme as in \cref{sec_exp_performance}.   }
	\label{fig_r}
\end{figure*}

To further test how the projection works, we investigate how the clustering performance of KTRR changes with respect to $r$ value.
Without loss of generality, we use Yale and PIX data sets for illustration.
For each data set, we consider two types of kernels with the same parameter settings as in previous test.
For each kernel, we vary $r\in\{1,3,5,7,9,11,13,15\}$.
For a fixed $r$ value we vary all the other parameters within the set $\{0.001,0.01,0.1,1,10,100,1000\}$, and we record the highest performance and report them in \cref{fig_r}.
It is seen that for each metric, two curves can be obtained corresponding to RBF and polynomial kernels, respectively. 
For both kernels, the performance of KTRR reaches the best performance with small $r$ in all metrics.
With large $r$ values, the performance of our method is not further improved, implying that a few number of projection directions can sufficiently extract key features of the data and lead to promising clustering performance. 
Thus, our method can also be applided as a powerful dimension reduction technique for 2D data.

\subsection{KTRR v.s. TRR}
In this test, we conduct some experiments on Yale and PIE data sets to verify the importance of learning nonlinear structures of data.
For Yale and PIE data sets, we use the same subsets as in \cref{sec_exp_performance}.
To show the importance of learning nonlinear structures with kernels, we compare the performances of KTRR with general kernels and linear kernel as two cases.
For the linear case, we use a linear kernel for KTRR and denote it as TRR.
For the other parameters, we tune them in the same way as in \cref{fig_ktrr_trr}.
For KTRR, we use general kernels as described in \cref{sec_methods} and tune the other parameters in the same way as TRR.
We report the best performances of KTRR as well as TRR with respect to the number of clusters in \cref{fig_ktrr_trr}.
It is seen that generally KTRR with general kernels outperforms TRR with linear kernel with significant improvements in many cases.
In fact, it is natural that KTRR can always perform no worse than TRR, because TRR is a special case of KTRR by using linear kernel and this ensures that KTRR has at least the same performance as TRR.
Generally speaking, we can observe much better performance if general kernels are used because they correspond to some more complicated nonlinear mappings, which may better capture nonlinear structures of the data than linear mapping.

\begin{figure*}[!]	
	\centering
	{\includegraphics[width=1.3\columnwidth]{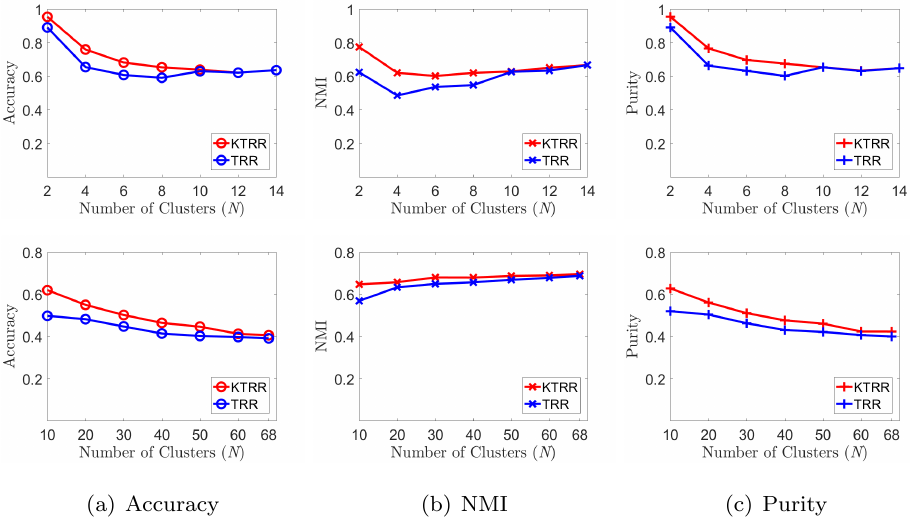}}
	\caption{  Examples of how the proposed method performs with general kernels (denoted as KTRR) and linear kernel (denoted as TRR) in accuracy, NMI, and purity on Yale (on the top) and PIE (on the bottom) data sets. 
		For a specific $N$ value, we report the best performance of KTRR and TRR by tuning all the other parameters in a grid-search scheme as in \cref{sec_exp_performance}.
	}
	\label{fig_ktrr_trr}
\end{figure*}

\section{Conclusions}
\label{sec_conclusion}
In this paper, we propose a novel subspace clustering method named KTRR for 2D data.
The KTRR provides us with a way, which is different from tensor methods, to learn the most representative 2D features from 2D data in learning data representation.
The KTRR performs 2D feature learning and low-dimensional representation construction simultaneously, which renders the two tasks to mutually enhance each other.
2D kernel renders the KTRR to have enhanced capability of capturing nonlinear relationships from data.
An efficient algorithm is developed for its optimization with provable decreasing and convergent property in objective value. 
Extensive experimental results confirm the effectiveness and efficiency of our method.

Besides the strengths of the KTRR, we should also note its weakness and possible further research directions, which are summarized as follows.
1) The KTRR captures spatial information from horizontal direction by multiplying a single projection matrix $P$ on right hand side,
which omits spatial information from vertical direction.
Thus, it is interesting to introduce another projection $Q$ on left hand side of the data examples for both horizontal and vertical spatial information extraction.
2) In KTRR, we need to provide a value for $r$, which determines the number of projection directions to seek.
After extending the KTRR to the bi-directional case, we need to provide the number of projections to seek from both sides.
It is interesting to develop the KTRR such that it can automatically determine the optimal number of projection directions for $P$ and $Q$, respectively, in a self-learning way.
3) For KTRR, the clustering performance relays on the kernel selection.
However, the optimal type of kernel function and parameters are not always available.
Thus, it is meaningful to develop multi-kernel model based on KTRR such that it can automatically learn an optimal kernel from a set of kernel functions. 

\section*{Acknowledgment}
This work is supported by National Natural Science Foundation of China (NSFC) under Grants 61806106, 61802215, and 61806045, Shandong Provincial Natural Science Foundation, China under Grants ZR2019QF009, and ZR2019BF011; Q.C. is supported by NIH UH3 NS100606-03.

\bibliographystyle{ieeetran}
\bibliography{kernel_trr}

\begin{thebibliography}{10}
\providecommand{\url}[1]{#1}
\csname url@samestyle\endcsname
\providecommand{\newblock}{\relax}
\providecommand{\bibinfo}[2]{#2}
\providecommand{\BIBentrySTDinterwordspacing}{\spaceskip=0pt\relax}
\providecommand{\BIBentryALTinterwordstretchfactor}{4}
\providecommand{\BIBentryALTinterwordspacing}{\spaceskip=\fontdimen2\font plus
\BIBentryALTinterwordstretchfactor\fontdimen3\font minus
  \fontdimen4\font\relax}
\providecommand{\BIBforeignlanguage}[2]{{%
\expandafter\ifx\csname l@#1\endcsname\relax
\typeout{** WARNING: IEEEtran.bst: No hyphenation pattern has been}%
\typeout{** loaded for the language `#1'. Using the pattern for}%
\typeout{** the default language instead.}%
\else
\language=\csname l@#1\endcsname
\fi
#2}}
\providecommand{\BIBdecl}{\relax}
\BIBdecl

\bibitem{peng2015subspace}
C.~Peng, Z.~Kang, H.~Li, and Q.~Cheng, ``Subspace clustering using
  log-determinant rank approximation,'' in \emph{Proceedings of the 21th ACM
  SIGKDD International Conference on Knowledge Discovery and Data
  Mining}.\hskip 1em plus 0.5em minus 0.4em\relax ACM, 2015, pp. 925--934.

\bibitem{liu2010robust}
G.~Liu, Z.~Lin, and Y.~Yu, ``Robust subspace segmentation by low-rank
  representation,'' in \emph{Proceedings of the 27th International Conference
  on Machine Learning (ICML-10)}, 2010, pp. 663--670.

\bibitem{boult1991factorization}
T.~E. Boult and L.~G. Brown, ``Factorization-based segmentation of motions,''
  in \emph{Visual Motion, 1991., Proceedings of the IEEE Workshop on}.\hskip
  1em plus 0.5em minus 0.4em\relax IEEE, 1991, pp. 179--186.

\bibitem{vidal2003generalized}
R.~Vidal, Y.~Ma, and S.~Sastry, ``Generalized principal component analysis
  (gpca),'' in \emph{Computer Vision and Pattern Recognition, 2003.
  Proceedings. 2003 IEEE Computer Society Conference on}, vol.~1.\hskip 1em
  plus 0.5em minus 0.4em\relax IEEE, 2003, pp. I--621.

\bibitem{ma2008estimation}
Y.~Ma, A.~Y. Yang, H.~Derksen, and R.~Fossum, ``Estimation of subspace
  arrangements with applications in modeling and segmenting mixed data,''
  \emph{SIAM review}, vol.~50, no.~3, pp. 413--458, 2008.

\bibitem{gruber2004multibody}
A.~Gruber and Y.~Weiss, ``Multibody factorization with uncertainty and missing
  data using the em algorithm,'' in \emph{Computer Vision and Pattern
  Recognition, 2004. CVPR 2004. Proceedings of the 2004 IEEE Computer Society
  Conference on}, vol.~1.\hskip 1em plus 0.5em minus 0.4em\relax IEEE, 2004,
  pp. I--707.

\bibitem{rao2010motion}
S.~Rao, R.~Tron, R.~Vidal, and Y.~Ma, ``Motion segmentation in the presence of
  outlying, incomplete, or corrupted trajectories,'' \emph{Pattern Analysis and
  Machine Intelligence, IEEE Transactions on}, vol.~32, no.~10, pp. 1832--1845,
  2010.

\bibitem{ma2007segmentation}
Y.~Ma, H.~Derksen, W.~Hong, and J.~Wright, ``Segmentation of multivariate mixed
  data via lossy data coding and compression,'' \emph{Pattern Analysis and
  Machine Intelligence, IEEE Transactions on}, vol.~29, no.~9, pp. 1546--1562,
  2007.

\bibitem{ho2003clustering}
J.~Ho, M.-H. Yang, J.~Lim, K.-C. Lee, and D.~Kriegman, ``Clustering appearances
  of objects under varying illumination conditions,'' in \emph{Computer Vision
  and Pattern Recognition, 2003. Proceedings. 2003 IEEE Computer Society
  Conference on}, vol.~1.\hskip 1em plus 0.5em minus 0.4em\relax IEEE, 2003,
  pp. I--11.

\bibitem{zhang2009median}
T.~Zhang, A.~Szlam, and G.~Lerman, ``Median k-flats for hybrid linear modeling
  with many outliers,'' in \emph{Computer Vision Workshops (ICCV Workshops),
  2009 IEEE 12th International Conference on}.\hskip 1em plus 0.5em minus
  0.4em\relax IEEE, 2009, pp. 234--241.

\bibitem{CHEN2018107}
H.~Chen, W.~Wang, and X.~Feng, ``Structured sparse subspace clustering with
  within-cluster grouping,'' \emph{Pattern Recognition}, vol.~83, pp. 107 --
  118, 2018.

\bibitem{SUI2019261}
Y.~Sui, G.~Wang, and L.~Zhang, ``Sparse subspace clustering via low-rank
  structure propagation,'' \emph{Pattern Recognition}, vol.~95, pp. 261 -- 271,
  2019.

\bibitem{CHEN2020107441}
Y.~Chen, X.~Xiao, and Y.~Zhou, ``Multi-view subspace clustering via
  simultaneously learning the representation tensor and affinity matrix,''
  \emph{Pattern Recognition}, vol. 106, p. 107441, 2020.

\bibitem{vidal2011subspace}
R.~Vidal, ``Subspace clustering,'' \emph{Signal Processing Magazine, IEEE},
  vol.~28, no.~2, pp. 52--68, March 2011.

\bibitem{liu2013robust}
G.~Liu, Z.~Lin, S.~Yan, J.~Sun, Y.~Yu, and Y.~Ma, ``Robust recovery of subspace
  structures by low-rank representation,'' \emph{Pattern Analysis and Machine
  Intelligence, IEEE Transactions on}, vol.~35, no.~1, pp. 171--184, 2013.

\bibitem{favaro2011closed}
P.~Favaro, R.~Vidal, and A.~Ravichandran, ``A closed form solution to robust
  subspace estimation and clustering,'' in \emph{Computer Vision and Pattern
  Recognition (CVPR), 2011 IEEE Conference on}.\hskip 1em plus 0.5em minus
  0.4em\relax IEEE, 2011, pp. 1801--1807.

\bibitem{elhamifar2013sparse}
E.~Elhamifar and R.~Vidal, ``Sparse subspace clustering: Algorithm, theory, and
  applications,'' \emph{Pattern Analysis and Machine Intelligence, IEEE
  Transactions on}, vol.~35, no.~11, pp. 2765--2781, 2013.

\bibitem{BRBIC2018247}
M.~Brbić and I.~Kopriva, ``Multi-view low-rank sparse subspace clustering,''
  \emph{Pattern Recognition}, vol.~73, pp. 247 -- 258, 2018.

\bibitem{peng2016feature}
C.~Peng, Z.~Kang, M.~Yang, and Q.~Cheng, ``Feature selection embedded subspace
  clustering,'' \emph{IEEE Signal Processing Letters}, vol.~PP, no.~99, pp.
  1--1, 2016.

\bibitem{patel2013latent}
V.~M. Patel, H.~Van~Nguyen, and R.~Vidal, ``Latent space sparse subspace
  clustering,'' in \emph{Proceedings of the IEEE International Conference on
  Computer Vision}, 2013, pp. 225--232.

\bibitem{peng2017integrating}
C.~Peng, Z.~Kang, and Q.~Cheng, ``Integrating feature and graph learning with
  low-rank representation,'' \emph{Neurocomputing}, vol. 249, pp. 106--116,
  2017.

\bibitem{liu2014enhancing}
J.~Liu, Y.~Chen, J.~Zhang, and Z.~Xu, ``Enhancing low-rank subspace clustering
  by manifold regularization,'' \emph{Image Processing, IEEE Transactions on},
  vol.~23, no.~9, pp. 4022--4030, 2014.

\bibitem{xiao2016robust}
S.~Xiao, M.~Tan, D.~Xu, and Z.~Y. Dong, ``Robust kernel low-rank
  representation,'' \emph{IEEE transactions on neural networks and learning
  systems}, vol.~27, no.~11, pp. 2268--2281, 2016.

\bibitem{patel2014kernel}
V.~M. Patel and R.~Vidal, ``Kernel sparse subspace clustering,'' in \emph{2014
  IEEE International Conference on Image Processing (ICIP)}.\hskip 1em plus
  0.5em minus 0.4em\relax IEEE, 2014, pp. 2849--2853.

\bibitem{peng2015robust}
X.~Peng, Z.~Yi, and H.~Tang, ``Robust subspace clustering via thresholding
  ridge regression.'' in \emph{AAAI}, 2015, pp. 3827--3833.

\bibitem{Xu2019Scaled}
J.~{Xu}, M.~{Yu}, L.~{Shao}, W.~{Zuo}, D.~{Meng}, L.~{Zhang}, and D.~{Zhang},
  ``Scaled simplex representation for subspace clustering,'' \emph{IEEE
  Transactions on Cybernetics}, pp. 1--13, 2019.

\bibitem{fu2016tensor}
Y.~Fu, J.~Gao, D.~Tien, Z.~Lin, and X.~Hong, ``Tensor lrr and sparse
  coding-based subspace clustering,'' \emph{IEEE transactions on neural
  networks and learning systems}, vol.~27, no.~10, pp. 2120--2133, 2016.

\bibitem{Benaroya2018Binaural}
L.~Benaroya, N.~Obin, M.~Liuni, A.~Roebel, W.~Raumel, and S.~Argentieri,
  ``Binaural localization of multiple sound sources by non-negative tensor
  factorization,'' \emph{IEEE/ACM Transactions on Audio Speech and Language
  Processing}, vol.~PP, no.~99, pp. 1--1, 2018.

\bibitem{LIU2020107252}
Y.~Liu, T.~Liu, J.~Liu, and C.~Zhu, ``Smooth robust tensor principal component
  analysis for compressed sensing of dynamic mri,'' \emph{Pattern Recognition},
  vol. 102, p. 107252, 2020.

\bibitem{lu2016tensor}
C.~Lu, J.~Feng, Y.~Chen, W.~Liu, Z.~Lin, and S.~Yan, ``Tensor robust principal
  component analysis: Exact recovery of corrupted low-rank tensors via convex
  optimization,'' in \emph{Proceedings of the IEEE Conference on Computer
  Vision and Pattern Recognition}, 2016, pp. 5249--5257.

\bibitem{Pan2019Tensor}
P.~Zhou, C.~Lu, J.~Feng, Z.~Lin, and S.~Yan, ``Tensor low-rank representation
  for data recovery and clustering,'' \emph{IEEE Transactions on Pattern
  Analysis and Machine Intelligence}, vol.~PP, no.~99, pp. 1--1, 2019.

\bibitem{zhang2015low}
C.~Zhang, H.~Fu, S.~Liu, G.~Liu, and X.~Cao, ``Low-rank tensor constrained
  multiview subspace clustering,'' in \emph{Proceedings of the IEEE
  International Conference on Computer Vision}, 2015, pp. 1582--1590.

\bibitem{Zhou2019Tensor}
P.~Zhou, L.~Canyi, J.~Feng, Z.~Lin, and S.~Yan, ``Tensor low-rank
  representation for data recovery and clustering,'' \emph{IEEE Transactions on
  Pattern Analysis and Machine Intelligence}, vol.~PP, pp. 1--1, 11 2019.

\bibitem{kolda2009tensor}
T.~G. Kolda and B.~W. Bader, ``Tensor decompositions and applications,''
  \emph{SIAM review}, vol.~51, no.~3, pp. 455--500, 2009.

\bibitem{letexier2008noise}
D.~Letexier and S.~Bourennane, ``Noise removal from hyperspectral images by
  multidimensional filtering,'' \emph{IEEE Transactions on Geoscience and
  Remote Sensing}, vol.~46, no.~7, pp. 2061--2069, 2008.

\bibitem{yang2004two}
{Jian Yang}, D.~{Zhang}, A.~F. {Frangi}, and {Jing-yu Yang}, ``Two-dimensional
  pca: a new approach to appearance-based face representation and
  recognition,'' \emph{IEEE Transactions on Pattern Analysis and Machine
  Intelligence}, vol.~26, no.~1, pp. 131--137, Jan 2004.

\bibitem{Peng2020TwoDimensionalSM}
C.~Peng, Z.~Zhang, Z.~Kang, C.~Chen, and Q.~Cheng, ``Two-dimensional
  semi-nonnegative matrix factorization for clustering,'' \emph{ArXiv}, vol.
  abs/2005.09229, 2020.

\bibitem{zhang2014nuclear}
F.~{Zhang}, J.~{Yang}, J.~{Qian}, and Y.~{Xu}, ``Nuclear norm-based 2-dpca for
  extracting features from images,'' \emph{IEEE Transactions on Neural Networks
  and Learning Systems}, vol.~26, no.~10, pp. 2247--2260, Oct 2015.

\bibitem{Yin2016Laplacian}
M.~Yin, J.~Gao, and Z.~Lin, ``s,'' \emph{IEEE Transactions on Pattern Analysis
  and Machine Intelligence}, vol.~38, no.~3, pp. 504--517, 2016.

\bibitem{liu2011latent}
G.~Liu and S.~Yan, ``Latent low-rank representation for subspace segmentation
  and feature extraction,'' in \emph{Computer Vision (ICCV), 2011 IEEE
  International Conference on}.\hskip 1em plus 0.5em minus 0.4em\relax IEEE,
  2011, pp. 1615--1622.

\bibitem{peng2020discriminative}
C.~{Peng} and Q.~{Cheng}, ``Discriminative ridge machine: A classifier for
  high-dimensional data or imbalanced data,'' \emph{IEEE Transactions on Neural
  Networks and Learning Systems}, pp. 1--15, 2020.

\bibitem{elhamifar2009sparse}
E.~Elhamifar and R.~Vidal, ``Sparse subspace clustering,'' in \emph{Computer
  Vision and Pattern Recognition, 2009. CVPR 2009. IEEE Conference on}.\hskip
  1em plus 0.5em minus 0.4em\relax IEEE, 2009, pp. 2790--2797.

\bibitem{nhat2007kernel}
V.~D.~M. Nhat and S.~Lee, ``Kernel-based 2dpca for face recognition,'' in
  \emph{Signal Processing and Information Technology, 2007 IEEE International
  Symposium on}.\hskip 1em plus 0.5em minus 0.4em\relax IEEE, 2007, pp. 35--39.

\bibitem{shi2000normalized}
J.~Shi and J.~Malik, ``Normalized cuts and image segmentation,'' \emph{IEEE
  Transactions on pattern analysis and machine intelligence}, vol.~22, no.~8,
  pp. 888--905, 2000.

\bibitem{agarwal2004k}
P.~K. Agarwal and N.~H. Mustafa, ``k-means projective clustering,'' in
  \emph{Proceedings of the twenty-third ACM SIGMOD-SIGACT-SIGART symposium on
  Principles of database systems}.\hskip 1em plus 0.5em minus 0.4em\relax ACM,
  2004, pp. 155--165.

\bibitem{li2015structured}
C.-G. Li and R.~Vidal, ``Structured sparse subspace clustering: A unified
  optimization framework,'' in \emph{Proceedings of the IEEE Conference on
  Computer Vision and Pattern Recognition}, 2015, pp. 277--286.

\bibitem{pan2017deep}
P.~Ji, T.~Zhang, H.~Li, M.~Salzmann, and I.~Reid, ``Deep subspace clustering
  networks,'' in \emph{Advances in Neural Information Processing Systems 30},
  I.~Guyon, U.~V. Luxburg, S.~Bengio, H.~Wallach, R.~Fergus, S.~Vishwanathan,
  and R.~Garnett, Eds.\hskip 1em plus 0.5em minus 0.4em\relax Curran
  Associates, Inc., 2017, pp. 24--33.

\bibitem{lyons1998japanese}
M.~J. Lyons, S.~Akamatsu, M.~Kamachi, J.~Gyoba, and J.~Budynek, ``The japanese
  female facial expression (jaffe) database,'' 1998.

\bibitem{hond1997distinctive}
D.~Hond and L.~Spacek, ``Distinctive descriptions for face processing.'' in
  \emph{BMVC}, no. 0.2, 1997, pp. 0--4.

\bibitem{belhumeur1997eigenfaces}
P.~N. Belhumeur, J.~P. Hespanha, and D.~J. Kriegman, ``Eigenfaces vs.
  fisherfaces: Recognition using class specific linear projection,'' \emph{IEEE
  Transactions on pattern analysis and machine intelligence}, vol.~19, no.~7,
  pp. 711--720, 1997.

\bibitem{samaria1994parameterisation}
F.~S. Samaria and A.~C. Harter, ``Parameterisation of a stochastic model for
  human face identification,'' in \emph{Proceedings of the Second IEEE Workshop
  on Applications of Computer Vision, 1994.}\hskip 1em plus 0.5em minus
  0.4em\relax IEEE, 1994, pp. 138--142.

\bibitem{peng2018integrate}
C.~Peng, Z.~Kang, S.~Cai, and Q.~Cheng, ``Integrate and conquer: Double-sided
  two-dimensional k-means via integrating of projection and manifold
  construction,'' \emph{ACM Transactions on Intelligent Systems and Technology
  (TIST)}, vol.~9, no.~5, p.~57, 2018.

\bibitem{peng2017nonnegative}
C.~Peng, Z.~Kang, Y.~Hu, J.~Cheng, and Q.~Cheng, ``Nonnegative matrix
  factorization with integrated graph and feature learning,'' \emph{ACM
  Transactions on Intelligent Systems and Technology (TIST)}, vol.~8, no.~3,
  p.~42, 2017.

\bibitem{peng2020nonnegative}
\BIBentryALTinterwordspacing
C.~Peng, Z.~Kang, C.~Chen, and Q.~Cheng, ``Nonnegative matrix factorization
  with local similarity learning,'' \emph{CoRR}, vol. abs/1907.04150, 2019.
  [Online]. Available: \url{http://arxiv.org/abs/1907.04150}
\BIBentrySTDinterwordspacing

\end{thebibliography}

\end{document}